\documentclass[conference]{IEEEtran}
\usepackage{times}  
\usepackage{helvet} 
\usepackage{courier}  
\usepackage[hyphens]{url}  
\usepackage{graphicx} 
\usepackage{graphicx}  

\usepackage{booktabs}       
\usepackage{amsfonts}       
\usepackage{amsthm}
\usepackage{algorithm}
\usepackage{amssymb}
\usepackage{graphicx}
\usepackage{subcaption,epsfig,epstopdf,color}
\usepackage{amsmath,bm}
\usepackage{algpseudocode}%
\usepackage{textcomp}
\usepackage{xcolor}
\newtheorem{definition}{Definition}
\newtheorem{lemma}{Lemma}
\newtheorem{assumption}{Assumption}
\newtheorem{remark}{Remark}
\newtheorem{theorem}{Theorem}

\ifCLASSINFOpdf
\else
\fi
\begin{document}
%
\title{Differentially Private $\ell_1$-norm Linear Regression with Heavy-tailed Data}

%
\author{\IEEEauthorblockN{Di Wang\IEEEauthorrefmark{1}  and 
Jinhui Xu\IEEEauthorrefmark{3}}
\IEEEauthorblockA{\IEEEauthorrefmark{1}Division of Computer, Electrical and Mathematical Sciences and Engineering\\
King Abdullah University of Science and Technology,
Thuwal, Saudi Arabia. Email: di.wang@kaust.edu.sa}
\IEEEauthorblockA{\IEEEauthorrefmark{3}Department of Computer Science and Engineering\\ State University of New York at Buffalo, Buffalo, NY. 
Email: jinhui@buffalo.edu}}


\maketitle

\begin{abstract}
We study the problem of Differentially Private Stochastic Convex Optimization (DP-SCO) with heavy-tailed data. Specifically, 
we focus on the $\ell_1$-norm linear regression in the $\epsilon$-DP model. While most of the previous work focuses on the case where the loss function is Lipschitz, here we only need to assume the variates has bounded moments. Firstly, we study the case where the $\ell_2$ norm of data has bounded second order moment. We propose an algorithm which is based on the exponential mechanism and show that it is possible to achieve an upper bound of $\tilde{O}(\sqrt{\frac{d}{n\epsilon}})$ (with high probability). Next, we relax the assumption to  bounded $\theta$-th order moment with some $\theta\in (1, 2)$ and show that it is possible to achieve an upper bound of $\tilde{O}(({\frac{d}{n\epsilon}})^\frac{\theta-1}{\theta})$. Our algorithms can also
be extended to more relaxed cases where only each coordinate
of the data has bounded moments, and we can get an upper bound of $\tilde{O}({\frac{d}{\sqrt{n\epsilon}}})$ and $\tilde{O}({\frac{d}{({n\epsilon})^\frac{\theta-1}{\theta}}})$ in the second and $\theta$-th moment case  respectively. 
\end{abstract}


%
\IEEEpeerreviewmaketitle

\section{Introduction}
As one of the most fundamental problem in both statistical machine learning and differential privacy (DP). Stochastic Convex Optimization under the differential privacy \cite{dwork2006calibrating} constraint, {\em i.e.,} Differentially Private Stochastic Convex Optimization (DPSCO), has received much attentions in recent years \cite{bassily2014private,wang2019differentially2,wang2019differentially,bassily2020,bassily2020stability,feldman2020private,song2020characterizing,su2021faster,asi2021private,bassily2021non,kulkarni2021private,bassily2021differentially,asi2021private2,DBLP:conf/ijcai/XueYH021}. In DPSCO, we have a loss function $\ell$ and an $n$-size dataset $D=\{(x_1, y_1), (x_2, y_2), \cdots, (x_n, y_n)\}$ where each pair of the variate and the label/response $(x_i, y_i)$ is i.i.d. sampled from some unknown distribution $\mathcal{D}$. The goal of DPSCO is to privately minimize the population risk function $L_\mathcal{D}(w)=\mathbb{E}_{(x, y)\sim \mathcal{D}}[\ell(w, x, y))]$ over a parameter space $\mathcal{W}\subseteq \mathbb{R}^d$, {\em i.e.,} we aim to design some DP algorithm whose output $w^{priv}$ makes the excess population risk  $L_\mathcal{D}(w^{priv})-\min_{w\in \mathcal{W}} L_\mathcal{D}(w^{priv})$ to be as small as possible with high probability. 

Although DPSCO and it empirical form, differentially Private Empirical Risk Minimization (DPERM), have been extensively studied for many years and there is a long of work attacked the problems from different perspectives, most of those work needs to assume the data distribution is bounded which indicates that the loss function is Lipschitz, which is unrealistic and does not always hold in practice. To relax the Lipschitz assumption, start from \cite{wang2020differentially}, there are several work have studied DPSCO with heavy-tailed data recently, where  the Lipschitz assumption is removed and we only need to assume that the distribution of the gradient of the loss function has bounded finite order of moments instead \cite{wang2020differentially,hu2021high,kamath2021improved}.  However, there are still several open problems in DP-SCO with heavy-tailed data. For example,  previous work only considered the case where the loss function is smooth. Moreover, most of those work studied the $(\epsilon, \delta)$-DP model and its behaviors in the $\epsilon$-DP model is still far from well-understood. Thirdly, all the previous results need to assume the data or the gradient of the loss  has at least bounded second (order) moments and cannot be extended to a more relaxed case where it has only $\theta$-th moment with $\theta\in (1, 2)$. In this paper, we  continue along the directions of these open problems. Specifically, we study the problem of  $\ell_1$-norm linear regression ({\em i.e., $\ell(w, x, y)=|\langle x, w\rangle-y|$}) with heavy-tailed data in $\epsilon$-DP model. Our contributions can be summarized as follows. 
\begin{itemize}
    \item We first consider the case where the second moment of the $\ell_2$-norm of the variate $x$, {\em i.e., } $\mathbb{E}\|x\|^2_2$, is bounded. Specifically, we propose a method which is based on the exponential mechanism and show that it is possible to achieve an upper bound of $\tilde{O}(\sqrt{\frac{d}{n\epsilon}})$ with high probability. Moreover, instead of the $\ell_2$-norm, we also consider the case where each coordinate of $x$ has bounded second moment, {\em i.e.,} $\mathbb{E}|x_j|^2< \infty$ for every $j\in [d]$. We show that our algorithm could achieve an error bound of $\tilde{O}({\frac{d}{\sqrt{n\epsilon}}})$. 
    \item We then investigate a relaxed case where the data only has $\theta$-th moment with $\theta\in (1, 2)$. First, similar to the second moment case, we assume that  $\mathbb{E}\|x\|^\theta_2<\infty$ and show it is possible to achieve a rate of $\tilde{O}\big(({\frac{d}{n\epsilon}})^\frac{\theta-1}{\theta}\big )$. Then, under the relaxed condition that $\mathbb{E}|x_j|^\theta< \infty$ for every $j\in [d]$, we show that our algorithm could achieve an error of $\tilde{O}\big(\frac{d}{(n\epsilon)^\frac{\theta-1}{\theta}}\big)$. To the best of our knowledge, this is the first theoretical result of DPSCO with heavy-tailed data that only has $\theta$-th moment with $\theta\in (1, 2)$.
\end{itemize}
\section{Related Work}
Although there is a long list of work studied either DPSCO/DPERM or robust estimation. DPSCO with heavy-tailed data is not well-understood. Below we will mentioned the related work on DPSCO with heavy-tailed data and private and robust mean estimation. 

For private estimation for heavy-tailed distribution,  \cite{barber2014privacy} provides the first study on private mean estimation for distributions with bounded second moment and proposes the minimax private rates.  Later, \cite{kamath2020private} also studies the heavy-tailed mean estimation with a relaxed assumption, which is also studied by \cite{liu2021robust,tao2021optimal} recently. However,  due to the complex nature of $\ell_1$ regression, the methods for mean estimation  cannot be used to our problem. 
Moreover, it is unknown whether their methods could be extended to the case where each coordinate of the data has only $\theta$-th order moment with $\theta\in (1, 2)$. 

For DPSCO with heavy-tailed data, \cite{wang2020differentially} first studies the problem by proposing three methods based on different assumptions. The first method is based on the smooth sensitivity and the Sample-and-Aggregate framework \cite{nissim2007smooth}. However, it needs enormous assumptions and its error bound is quite large. Their second method is still based on the smooth sensitivity \cite{bun2019average}. However, it needs to assume the distribution is sub-exponential.\cite{wang2020differentially} also provides a new private estimator motivated by the previous work in robust statistics. \cite{kamath2021improved} recently revisits the problem and improves the (expected) excess population risk for both convex and strongly convex loss functions. It also provides the lower bounds of the mean estimation problem in both $(\epsilon, \delta)$-DP and $\epsilon$-DP models. However, as we mentioned earlier, all of those algorithms are for $(\epsilon, \delta)$-DP model and need to assume the loss function is smooth. Thus, their methods cannot be used to our problem. 
Later, \cite{hu2021high} studies the problem in the high dimensional space where the dimension could far greater than the data size. It focuses the case where the loss function is smooth and the constraint set is some polytope or some $\ell_0$-norm ball, which is quite different with our settings. 

\section{Preliminaries}
\begin{definition}[Differential Privacy \cite{dwork2006calibrating}]\label{def:3.1}
	Given a data universe $\mathcal{X}$, we say that two datasets $D,D'\subseteq \mathcal{X}$ are neighbors if they differ by only one entry, which is denoted as $D \sim D'$. A randomized algorithm $\mathcal{A}$ is $\epsilon$-differentially private (DP) if for all neighboring datasets $D,D'$ and for all events $S$ in the output space of $\mathcal{A}$, the following holds
	\(\text{Pr}(\mathcal{A}(D)\in S)\leq e^{\epsilon} \text{Pr}(\mathcal{A}(D')\in S).\)
\end{definition}
	\begin{definition}[Exponential Mechanism]
		The Exponential Mechanism allows differentially private computation over arbitrary domains and range $\mathcal{R}$, parametrized by a score function $u(D,r)$ which maps a pair of input data set $D$ and candidate result $r\in \mathcal{R}$ to a real valued score. With the score function $u$ and privacy budget $\epsilon$, the mechanism yields an output with exponential bias in favor of high scoring outputs. Let $\mathcal{M}(D,x,R)$ denote the exponential mechanism, and $\Delta$ be the sensitivity of $u$ in the range $R$, 
		$\Delta=\max_{r\in \mathcal{R}}\max_{D\sim D'}|u(D,r)-u(D',r)|.$
		Then if $\mathcal{M}(D,x,R)$ selects and outputs an element $r\in \mathcal{R}$ with probability proportional to $\exp(\frac{\epsilon u(D,r)}{2\Delta u})$, it  preserves $\epsilon$-differential privacy.
	\end{definition}
		\begin{lemma}\cite{dwork2014algorithmic}\label{lemma:exp}
		For the exponential mechanism $\mathcal{M}(D,u,\mathcal{R})$, we have 
		\begin{equation*}
		\text{Pr}\{u(\mathcal{M}(D,u,\mathcal{R}))\leq OPT_u(x)-\frac{2\Delta u}{\epsilon}(\ln|\mathcal{R}|+t)\}\leq e^{-t}.
		\end{equation*}
		where $OPT_u(x)$ is the highest score in the range $\mathcal{R}$, {\em i.e.} $\max_{r\in \mathcal{R}}u(D,r)$.
	\end{lemma}
	\begin{definition}[DPSCO \cite{bassily2014private}]\label{def:3}
		Given a dataset $D=\{z_1,\cdots,z_n\}$ from a data universe $\mathcal{Z}$ where $z_i=(x_i, y_i)$ with a feature vector $x_i$ and a label/response $y_i$ are i.i.d. samples from some unknown distribution $\mathcal{D}$,  a convex constraint set  $\mathcal{W} \subseteq \mathbb{R}^d$, and a convex loss function $\ell: \mathcal{W}\times \mathcal{Z}\mapsto \mathbb{R}$. Differentially Private Stochastic Convex Optimization (DPSCO) is to find $w^{\text{priv}}$ so as to minimize the population risk, {\em i.e.,} $L_\mathcal{D} (w)=\mathbb{E}_{z\sim \mathcal{D}}[\ell(w, z)]$
		with the guarantee of being DP.
		 The utility of the algorithm is measured by the excess population risk, that is  $L_\mathcal{D} (w^{\text{priv}})-\min_{w\in \mathbb{\mathcal{W}}}L_\mathcal{D} (w)$ (for convenience we denote the optimal solution as $w^*$). Besides the population risk, we can also measure the \textit{empirical risk} of dataset $D$: $\hat{L}(w, D)=\frac{1}{n}\sum_{i=1}^n \ell(w, z_i).$ It is notable that in the \textbf{high probability setting}, we need to get a high probability excess population risk.  That is given a failure probability $0<\eta<1$, we want get a (polynomial) function $f(d, \log\frac{1}{\eta},  \frac{1}{n}, \frac{1}{\epsilon})$ such that with probability at least $1-\eta$ (over the randomness of the algorithm and the data distribution), $$L_\mathcal{D} (w^{\text{priv}})-L_\mathcal{D} (w^*)\leq O(f(d,  \log\frac{1}{\eta}, \frac{1}{n}, \frac{1}{\epsilon})).$$ 
	\end{definition}
In this paper, we mainly focus on $\ell_1$-norm linear regression:  
\begin{equation}\label{eq:19}
\min_{w\in \mathcal{W}} L_\mathcal{D}(w)=\mathbb{E}_{(x,y)\sim \mathcal{D}}|\langle x, w\rangle-y|, 
\end{equation}
where the convex constraint set $\mathcal{W}$ is bounded with diameter $\Delta=\max_{w, w'\in \mathcal{W}}\|w-w'\|_2< \infty$.
\begin{definition}[$\zeta$-Net]\label{def:4}
Let $(T, d)$ be a metric space. Consider a subset $\mathcal{W} \subset T$ and let $\zeta>0$. A subset $\mathcal{S}\subseteq \mathcal{W}$ is called an $\zeta$-net of $\mathcal{W}$ if every point in $\mathcal{W}$ is within a distance $\zeta$ of some points of $\mathcal{S}$, {\em i.e.,}
    $\forall x\in k, \exists x_0\in \mathcal{N}: d(x, x_0)\leq\zeta. $ 
The  smallest possible cardinality of an $\zeta$-net of $\mathcal{W}$ is called the covering number of $\mathcal{W}$ and is denoted by $\mathcal{N}(\mathcal{W}, d, \zeta)$. Equivalently, covering number is the smallest number of closed balls with centers in $K$ and radii $\zeta$ whose union covers $\mathcal{W}$. 
\end{definition}
\begin{lemma}[\cite{vershynin2018high}]
    For the Euclidean space $(\mathbb{R}^d, \|\cdot\|_2)$, and a bounded subset $\mathcal{W}\subseteq \mathbb{R}^d$ with the diameter $\Delta$. Then its covering number  $\mathcal{N}(\mathcal{W},  \zeta)\leq (\frac{3\Delta }{\zeta})^d$. 
\end{lemma}
\section{Main Methods}
 \subsection{Bounded second moment case}
 In this section we consider the case of bounded second moment. As mentioned earlier, in the previous work on DPSCO with heavy-tailed data, we always assume the distribution of gradient has bounded moments \cite{wang2020differentially,hu2021high,kamath2021improved}, if the loss function is smooth. However, for $\ell_1$ regression,  here the loss function is non-differentiable. Thus, instead of the gradient, here we will directly assume the second moments of $x$ are bounded, which implies that the  second moments of the sub-gradient of the loss function are also bounded. In general,  there are two assumptions on the heavy-tailedness of $x$; one assumes the distribution of $\|x\|_2$ has bounded moment; the other one assumes the distribution of each coordinate of $x$ has bounded moment. Formally, 
\begin{assumption}\label{ass:6}
	The second moment of $\|x\|_2$ is bounded by $\tau^2=O(1)$, that is $\mathbb{E}_{(x,y)\sim \mathcal{D}} \|x\|_2^2 \leq \tau^2. $
\end{assumption}
\begin{assumption}\label{ass:2}
    The second moment of each coordinate of $x$ is bounded by $\tau^2=O(1)$, that is 
	$\forall j\in[d], \mathbb{E}_{(x,y)\sim \mathcal{D}} x_j^2 \leq \tau^2. $
\end{assumption}
From the above two assumptions we can see that Assumption \ref{ass:6} is more restricted than Assumption \ref{ass:2}. 
Before showing our algorithm, we first provide a brief overview on the approach of solving the problem (\ref{eq:19}) in the non-private case, proposed by \cite{zhang2018ell_1}. Specifically,  instead of study the empirical risk function of the population risk (\ref{eq:19}), \cite{zhang2018ell_1} considers the following minimization problem of a truncated  loss : 
\begin{equation}\label{eq:20}
	\min_{w\in \mathcal{W}} \hat{L}_\iota(w, D)=\frac{1}{n\iota}\sum_{i=1}^n \psi (\iota |y_i-\langle x_i ,w \rangle |),  
\end{equation}
where $\iota>0$ is a parameter that will be specified later, the truncation function $\psi(\cdot)$ is a non-decreasing function which should satisfies the following property: 
\begin{equation}\label{eq:21}
	-\log (1-x+\frac{x^2}{2}) \leq \psi(x) \leq \log (1+x+\frac{x^2}{2}).
	\end{equation}
Specifically, \cite{zhang2018ell_1} shows the following result: 
	\begin{lemma}[Corollary 2 in \cite{zhang2018ell_1}]
		Under Assumption \ref{ass:6} and assume the $\psi(\cdot)$ satisfies (\ref{eq:21}). Then for any given failure probability $\eta$, for some specified $\iota=\iota(\frac{1}{n}, d, \Delta, \log \frac{1}{\eta})$ in (\ref{eq:20}), the optimal solution $\hat{w}_\iota$ of (\ref{eq:20}) has the following the excess population risk with probability at least $1-\eta$
		\begin{align}
			&L_\mathcal{D} (\hat{w}_\iota) - L_\mathcal{D}(w^*) \leq \tilde{O} \big(\frac{1}{n}\mathbb{E}\|x\|_2+\\
			& \sqrt{\frac{d\log \frac{1}{\eta}}{n}}(\frac{1}{n^2}+\sup_{w\in \mathcal{W}}\mathbb{E}(y-\langle x, w \rangle)^2 )\big)=\tilde{O}(\sqrt{\frac{d\log \frac{1}{\eta}}{n}}). 
		\end{align}
	\end{lemma}

	 The previous lemma indicates that solving the problem (\ref{eq:20}) is sufficient to solve the $\ell_1$ regression problem if 	$\sup_{w\in \mathcal{W}} \mathbb{E}|y-\langle x,  w\rangle|^2 =O(1)$. To solve the problem (\ref{eq:20}) differentially privately, we adapt the following specific form of $\psi(\cdot)$: 
	\begin{equation}
		\psi(x)= \begin{cases}\label{eq:22}
		-\log (1-x+\frac{x^2}{2}), &0\leq x\leq 1 \\
		\log 2, &x\geq 1\\
		-\psi(-x), &x\leq 0
	\end{cases}
\end{equation}
We can easily see (\ref{eq:22}) satisfies the property in (\ref{eq:21}). Moreover, due to the non-decreasing property we can see the function is upper bounded by  $\log 2$. That is, for a fixed $w$, the sensitivity of $\hat{L}_\iota(w, D)$ is bounded by $\frac{2\log 2}{n\iota}$. Motivated by this, we can use the exponential mechanism to solve (\ref{eq:20}) in $\epsilon$-DP model, see Algorithm \ref{alg:6} for details. 
\begin{algorithm}[!ht]
	\caption{Exponential mechanism for $\ell_1$-regression (second moment)}\label{alg:6}
	$\mathbf{Input}$: $D=\{(x_i, y_i)\}_{i=1}^n$; privacy parameter $\epsilon$; parameters $\iota, \zeta$ (will be specified later); truncated empirical risk $\hat{L}_\iota$ in (\ref{eq:20}) with $\psi$ in (\ref{eq:22}). 
	\begin{algorithmic}[1]
    \State Find a $\zeta$-net of $\mathcal{W}$ with covering number $N(\mathcal{W}, \zeta)$, denote it as $\tilde{W}_\zeta=\{w_1, \cdots, w_{N(\mathcal{W}, \zeta)}\}$. 
    \State Run the exponential mechanism  with the range $R=\tilde{W}_\zeta$ and the score function $u(D, w)=-\hat{L}_\iota(w, D)$.
     That is, output an $w\in \tilde{W}_\zeta$ with probability proportional to $\exp( \frac{-n\iota\epsilon \hat{L}_\iota(w, D)}{\log 2})$. 
	\end{algorithmic}
\end{algorithm}
\begin{theorem}\label{thm:9}
	For any $\epsilon>0$, Algorithm \ref{alg:6} is $\epsilon$-DP. Moreover, under Assumption \ref{ass:6},  given any failure probability $\eta\in (0,1)$, for the output $\tilde{w}$  we have the following with probability at least $1-\eta$ for any  $\iota>0$,
	\begin{equation*}
			L_\mathcal{D}(\hat{w}_\iota)-L_\mathcal{D}(w^*)\leq {O}(\zeta \tau+\iota \zeta^2\tau^2+\iota \tau^2\Delta^2 +\frac{1}{n\iota\epsilon}\log \frac{N(\mathcal{W}, \zeta)}{ \zeta}). 
	\end{equation*}
	 Furthermore, by setting $\zeta=O(\frac{1}{n})$ and $\iota=O(\sqrt{\frac{d\log n \log \frac{1}{\eta}}{n\epsilon\tau^2}})$ we have 
	 \begin{equation}\label{eq:24}
	 	L_\mathcal{D}(\hat{w}_\iota)-L_\mathcal{D}(w^*)\leq {O}( \tau\sqrt{\frac{d\log n\log \frac{1}{\eta}}{n\epsilon}}). 
	 \end{equation}
Moreover, under Assumption \ref{ass:2}, set $\zeta=O(\frac{1}{n})$ and $\iota=O(\sqrt{\frac{\log n \log \frac{1}{\eta}}{\tau^2 n\epsilon}})$ we have 
	 	 \begin{equation}\label{eq:9}
	 	L_\mathcal{D}(\hat{w}_\iota)-L_\mathcal{D}(w^*)\leq {O}(\tau\sqrt{\frac{d^2\log n\log \frac{1}{\eta}}{n\epsilon}}),
	 \end{equation}
	 	 where Big-${O}$ notations omit the term of $\Delta$. 
\end{theorem}
\begin{remark}
First, notice that since Assumption \ref{ass:6} is more stronger, there is a gap of $O(\sqrt{d})$ compared with (\ref{eq:24}) and (\ref{eq:9}). Moreover, for the upper bound in (\ref{eq:9}), it matches the lower bound of the private mean estimation under Assumption \ref{ass:2} in \cite{kamath2021improved}. However, it is still unknown whether this lower bound is optimal for DPSCO with  general  convex loss. To the best of our knowledge, this is the first $\epsilon$-DP algorithm which could achieve the bound of $\tilde{O}({\frac{d}{\sqrt{n\epsilon
}}})$ under the same assumption. 
\end{remark}
\begin{proof}[{\bf Proof of Theorem \ref{thm:9}}]
	The proof of $\epsilon$-DP is due to the sensitivity of the score function is bounded by $\frac{2\log 2}{n \iota}$ and the exponential mechanism. Next we will proof the utility. For convenience we denote $w_\zeta=\arg\min_{w\in \tilde{W}_\zeta} \hat{L}_\iota(w, D)$ and the optimal solution  of (\ref{eq:20}) as $\hat{w}_\iota$. By the utility of exponential mechanism Lemma \ref{lemma:exp} and take $t=\log \frac{1}{\eta}$ we have with probability at least $1-\eta$,  
		$-\hat{L}_\iota(\tilde{w}, D)\geq -\hat{L}_\iota(w_\zeta, D)-\frac{4\log 2}{n\iota \epsilon}\log \frac{N(\mathcal{W}, \zeta)}{\eta}.  $ 
	That is 
	
		\begin{equation}\label{aeq:50}
		\hat{L}_\iota(\tilde{w}, D)\leq  \hat{L}_\iota(w_\zeta, D)+\frac{4\log 2}{n\iota \epsilon}\log \frac{N(\mathcal{W}, \zeta)}{\eta}.  
	\end{equation} 
	For the term $\hat{L}_\iota(\tilde{w}, D)$ in (\ref{aeq:50}) we have the following. 
	
  \begin{lemma}[\cite{zhang2018ell_1}]\label{alemma:100}
 For any given $\eta$, we have the following inequality for all $w\in \tilde{W}_\zeta$ with probability at least $1-\eta$

 \begin{multline}
 	\hat{L}_\iota(w, D)=\frac{1}{n\iota}\sum_{i=1}^n \psi (\iota|y_i-\langle x_i, w \rangle|)\geq L_\mathcal{D}(w)\\ -\frac{\iota}{2}\sup_{w\in \mathcal{W}}\mathbb{E}(y-\langle x, w \rangle)^2)-\frac{1}{n\iota}\log \frac{N(\mathcal{W}, \zeta)}{\eta}. 
 \end{multline}	
 \end{lemma}

 For the term $\hat{L}_\iota(w_\zeta, D)$ in (\ref{aeq:50}), since $\tilde{W}_\zeta$ is a $\zeta$-net, thus there exists a $\tilde{w}_\iota \in \tilde{W}_\zeta$ such that $\|\tilde{w}_\iota-\hat{w}_\iota\|_2\leq \zeta$, where $\hat{w}_\iota=\arg\min_{w\in \mathcal{W}}\hat{L}_\iota (w, D)$. And by the definition we have 
 \begin{equation}\label{aeq:53}
 	\hat{L}_\iota(w_\zeta, D)\leq \hat{L}_\iota (\tilde{w}_\iota, D).
 \end{equation}
 For the term $\hat{L}_\iota (\tilde{w}_\iota, D)$ we have the following lemma 
 \begin{lemma}[\cite{zhang2018ell_1}]\label{alemma:101}
 For any given $\eta$, we have the following inequality for all $w\in \tilde{W}_\zeta$ with probability at least $1-\eta$
 	\begin{equation}
 		\hat{L}_\iota (w, D)\leq L_\mathcal{D}(w)+\frac{\iota}{2}\sup_{w\in\mathcal{W}}\mathbb{E}(y-\langle x, w\rangle )^2 +\frac{1}{n\iota} \log \frac{N(\mathcal{W}, \zeta)}{\eta}. 
 	\end{equation}
 \end{lemma}

 Thus, combing with (\ref{aeq:53}) and Lemma \ref{alemma:101} we have with probability at least $1-\eta$
 \begin{align}
 &	\hat{L}_\iota(w_\zeta, D)\leq \hat{L}_\iota (\tilde{w}_\iota, D)\notag \\
 	&\leq L_\mathcal{D}(\tilde{w}_\iota)+\frac{\iota}{2}\sup_{w\in\mathcal{W}}\mathbb{E}(y-\langle x, w\rangle )^2 +\frac{1}{n\iota} \log \frac{N(\mathcal{W}, \zeta)}{\eta} \nonumber \\
 	&\leq L_\mathcal{D}(\hat{w}_\iota)+\zeta\tau+\frac{\iota}{2}\sup_{w\in\mathcal{W}}\mathbb{E}(y-\langle x, w\rangle )^2 +\frac{1}{n\iota} \log \frac{N(\mathcal{W}, \zeta)}{\eta}, \label{aeq:56}
 \end{align}
 where the last inequality is due to 
 \begin{align*}
 	L_\mathcal{D}(\tilde{w}_\iota)-L_\mathcal{D}(\hat{w}_\iota)&=\mathbb{E}[|y-\langle x, \tilde{w}_\iota \rangle|-|y-\langle x, \hat{w}_\iota \rangle|]\\
 	&\leq \mathbb{E} |\langle x, \tilde{w}_\iota-\hat{w}_\iota\rangle|\leq \zeta\mathbb{E}\|x\|_2
 \end{align*}
 The relation between $L_\mathcal{D}(\hat{w}_\iota)$ and $L_\mathcal{D}(w^*)$ is due to the following lemma, given by \cite{zhang2018ell_1}, 
 \begin{lemma}\label{alemma:103}
 	For any given failure probability $\eta$, under Assumption \ref{ass:6}, we have the following with probability at least $1-2\eta$ 
 	\begin{multline*}
 		L_\mathcal{D}(\hat{w}_\iota)-L_\mathcal{D}(w^*)\leq 2\zeta \tau\\ +\iota \zeta^2\tau^2+\frac{3\iota}{2}\sup_{w\in\mathcal{W}}\mathbb{E}(y-\langle x, w\rangle)^2+\frac{1}{n\iota}\log \frac{N(\mathcal{W}, \zeta)}{\eta}. 
 	\end{multline*}
 	
 \end{lemma}
 Thus, taking Lemma \ref{alemma:100}, (\ref{aeq:56}) and Lemma \ref{alemma:103} into (\ref{aeq:50}) we have with probability at least $1-4\eta$
 \begin{align}
  &L_\mathcal{D}(\tilde{w})-L_\mathcal{D}(w^*) \leq 3\zeta \tau+\iota \zeta^2\tau^2 +\frac{5\iota}{2}\sup_{w\in\mathcal{W}}\mathbb{E}(y-\langle x, w\rangle)^2 \notag \\
  &+\frac{3}{n\iota}\log \frac{N(\mathcal{W}, \zeta)}{\eta} +\frac{4\log 2}{n\iota \epsilon}\log \frac{N(\mathcal{W}, \zeta)}{\eta}. \label{eq:140}
 \end{align}
By  $\log N(\mathcal{W}, \zeta)\leq d\log \frac{3\Delta}{\zeta}$ and the following inequality we can complete the proof. 
 \begin{multline*}
 	\sup_{w\in\mathcal{W}}\mathbb{E}(y-\langle x, w\rangle)^2\\ \leq \mathbb{E}(y-\langle x^*, w\rangle)^2+ 2\mathbb{E}\|x\|_2^2 \sup_{w\in\mathcal{W}}\|w-w^*\|_2^2=O(\Delta^2\tau^2).
 \end{multline*}
For (\ref{eq:9}), note that using the same proof we can replace $\tau$ by $\mathbb{E}\|x\|_2$ in (\ref{eq:140}). By Assumption \ref{ass:2} we have $\mathbb{E}\|x\|_2\leq \tau\sqrt{d}$ and $\mathbb{E}\|x\|^2_2\leq d\tau^2$. Thus, take $\zeta=O(\frac{1}{n})$ and $\iota=O(\sqrt{\frac{\log n \log \frac{1}{\eta}}{\tau^2 n\epsilon}})$ we finish the proof. 
\end{proof}
\subsection{Bounded $\theta$-th moment case}
Actually, motivated by \cite{chen2020generalized}, for the $\ell_1$-regression problem, we can even relax the second moment assumption in Assumption \ref{ass:6} and \ref{ass:2} to finite $\theta$-th moment assumptions with any $\theta\in (1, 2)$. 
Similar to the second moment case, here we consider two cases: 
\begin{assumption}\label{ass:7}
		There exists an $\theta\in (1, 2)$ such that the $\theta$-th order moment of $x$ is bounded by $\tau^\theta<\infty$ for some constant $\tau$, \footnote{Here we use $\tau^\theta$ is for convenience to compare with the second moment case.} that is 
		$\mathbb{E}_{(x,y)\sim \mathcal{D}} \|x\|_2^\theta \leq \tau^\theta=O(1).$ 
\end{assumption}
\begin{assumption}\label{ass:3}
    	We assume that the second moment of each coordinate of $x$ is bounded by $\tau^\theta=O(1)$, that is 
	$\forall j\in[d], \mathbb{E}_{(x,y)\sim \mathcal{D}} x_j^\theta \leq \tau^\theta=O(1). $
\end{assumption}

Here our main idea is almost the same as in the bounded second-order moment case and we will still focus on the function $\hat{L}_\iota (w, D)$ in (\ref{eq:20}). However, here we will adjust the non-decreasing truncation function $\psi: \mathbb{R}\mapsto \mathbb{R}$ to make it satisfies the following inequality instead of (\ref{eq:21}): 
\begin{equation}\label{eq:25}
	-\log (1-x+\frac{|x|^\theta}{\theta}) \leq \psi(x) \leq \log (1+x+\frac{|x|^\theta}{\theta}).
	\end{equation}
Motived by (\ref{eq:22}), here we use the following specific form for $\psi$: 
	\begin{equation}
		\psi(x)= \begin{cases}\label{eq:26}
		-\log (1-x+\frac{|x|^\theta}{\theta}), &0\leq x\leq 1 \\
		\log \theta, &x\geq 1\\
		-\psi(-x), &x\leq 0
	\end{cases}
\end{equation}

\begin{algorithm}[!ht]
	\caption{Exponential mechanism for $\ell_1$-regression ($\theta$-th moment)}\label{alg:7}
	$\mathbf{Input}$: $D=\{(x_i, y_i)\}_{i=1}^n$; privacy parameter $\epsilon$; parameters $\iota, \zeta$; truncated empirical risk $\hat{L}_\iota$ in (\ref{eq:20}) with $\psi$ in (\ref{eq:26}).
	\begin{algorithmic}[1]
    \State Find a $\zeta$-net of $\mathcal{W}$ with covering number $N(\mathcal{W}, \zeta)$, denote it as $\tilde{W}_\zeta=\{w_1, \cdots, w_{N(\mathcal{W}, \zeta)}\}$. 
    \State Run the exponential mechanism  with the range $R=\tilde{W}_\zeta$ and the score function $u(D, w)=-\hat{L}_\iota(w, D)$.
     That is, output an $w\in \tilde{W}_\zeta$ with probability proportional to $\exp( \frac{-n\iota\epsilon \hat{L}_\iota(w, D)}{\log \theta})$. 
	\end{algorithmic}
\end{algorithm}
We can easily see it satisfies the inequality in (\ref{eq:25}). Moreover, its absolute value is upper bounded by $\log \theta$. That is the sensitivity of $\hat{L}_\iota (w, D)$ is upper bounded by $\frac{2\log \theta}{2n\iota}$. Therefore, we will use the exponential mechanism (see Algorithm \ref{alg:7} for details) and its output has the following utility. 

\begin{theorem}\label{thm:10}
	For any $\epsilon>0$, Algorithm \ref{alg:7} is $\epsilon$-DP. Moreover,  under Assumption \ref{ass:7},  given failure probability $\zeta$, for the output  $\tilde{w}$ we have the following with probability at least $1-\zeta$ for any  $\iota>0$
 \begin{align} L_\mathcal{D}(\tilde{w})-  L_\mathcal{D}(w^*) &\leq O\big(\zeta \tau + {\iota^{\theta-1}}\tau^\theta+ \notag \\ & {\iota^{\theta-1}\zeta^\theta}\tau^\theta +\frac{1}{n\iota \epsilon}\log \frac{N(\mathcal{W}, \zeta)}{\eta}\big). 
 \end{align}
Take $\zeta=O(\frac{1}{n})$ and $\iota=O(\frac{1}{\tau}(\frac{d\log \frac{1}{\zeta}}{n\epsilon})^\frac{1}{\theta})$ we have 
	 \begin{equation}\label{eq:28}
	 	L_\mathcal{D}(\hat{w}_\iota)-L_\mathcal{D}(w^*)\leq \tilde{O}(\tau (\frac{d\log n\log \frac{1}{\zeta}}{n\epsilon})^\frac{\theta-1}{\theta}). 
	 \end{equation}
Moreover, under Assumption \ref{ass:3}, set $\zeta=O(\frac{1}{n})$ and $\iota=O(\frac{1}{\tau}(\frac{\log n\log \frac{1}{\eta}}{n\epsilon})^\frac{\theta-1}{\theta})$ we have
	 \begin{equation}\label{eq:180}
	 	L_\mathcal{D}(\hat{w}_\iota)-L_\mathcal{D}(w^*)\leq {O}(\tau d (\frac{\log n\log \frac{1}{\zeta}}{n\epsilon})^\frac{\theta-1}{\theta}). 
	 \end{equation}
\end{theorem}

\begin{proof}[{\bf Proof of Theorem \ref{thm:10}}]
	The proof of $\epsilon$-DP is due to the sensitivity of the score function is bounded by $\frac{2\log \theta}{n \iota}$ and the exponential mechanism. Next we will proof the utility. For convenience we denote $w_\zeta=\arg\min_{w\in \tilde{W}_\zeta} \hat{L}_\iota(w, D)$ and the optimal solution  of (\ref{eq:20}) as $\hat{w}_\iota$. By the utility of exponential mechanism Lemma \ref{lemma:exp} and take $t=\log \frac{1}{\eta}$ we have with probability at least $1-\eta$, 
		$-\hat{L}_\iota(\tilde{w}, D)\geq -\hat{L}_\iota(w_\zeta, D)-\frac{4\log \theta}{n\iota \epsilon}\log \frac{N(\mathcal{W}, \zeta)}{\eta}.$  
	That is 
	\begin{equation}\label{aeq:200}
		\hat{L}_\iota(\tilde{w}, D)\leq  \hat{L}_\iota(w_\zeta, D)+\frac{4\log \theta}{n\iota \epsilon}\log \frac{N(\mathcal{W}, \zeta)}{\eta}.  
	\end{equation} 
	For the term $\hat{L}_\iota(\tilde{w}, D)$ in (\ref{aeq:200}) we have the following inequality.
	
  \begin{lemma}[\cite{chen2020generalized}]\label{alemma:200}
 For any given $\zeta$, we have the following inequality for all $w\in \tilde{W}_\zeta$ with probability at least $1-\eta$

 \begin{align*}
 	&\hat{L}_\iota(w, D)=\frac{1}{n\iota}\sum_{i=1}^n \psi (\iota|y_i-\langle x_i, w \rangle|)\notag \\
 	& \geq L_\mathcal{D}(w)-\frac{\iota^{\theta-1}}{\theta}\sup_{w\in \mathcal{W}}\mathbb{E}|y-\langle x, w \rangle|^\theta-\frac{1}{n\iota}\log \frac{N(\mathcal{W}, \zeta)}{\eta}. 
 \end{align*}	
 \end{lemma}

 For the term $\hat{L}_\iota(w_\zeta, D)$ in (\ref{aeq:200}), since $\tilde{W}_\zeta$ is a $\zeta$-net, thus there exists a $\tilde{w}_\iota \in \tilde{W}_\zeta$ such that $\|\tilde{w}_\iota-\hat{w}_\iota\|_2\leq \zeta$, where $\hat{w}_\iota=\arg\min_{w\in \mathcal{W}}\hat{L}_\iota (w, D)$. And by the definition we have 
 \begin{equation}\label{aeq:203}
 	\hat{L}_\iota(w_\zeta, D)\leq \hat{L}_\iota (\tilde{w}_\iota, D).
 \end{equation}
 For the term $\hat{L}_\iota (\tilde{w}_\iota, D)$ we have the following lemma 
 \begin{lemma}\label{alemma:201}
 For any given $\eta$, we have the following inequality for all $w\in \tilde{W}_\zeta$ with probability at least $1-\eta$
 	\begin{equation}
 		\hat{L}_\iota (w, D)\leq L_\mathcal{D}(w)+\frac{\iota^{\theta-1}}{\theta}\sup_{w\in\mathcal{W}}\mathbb{E}|y-\langle x, w\rangle|^\theta +\frac{1}{n\iota} \log \frac{N(\mathcal{W}, \zeta)}{\eta}. 
 	\end{equation}
 \end{lemma}

 Thus, combing with (\ref{aeq:203}) and Lemma \ref{alemma:201} we have with probability at least $1-\eta$
 \begin{align}
 	&\hat{L}_\iota(w_\zeta, D)\leq \hat{L}_\iota (\tilde{w}_\iota, D)\notag \\
 	&\leq L_\mathcal{D}(\tilde{w}_\iota)+\frac{\iota^{\theta-1}}{\theta}\sup_{w\in\mathcal{W}}\mathbb{E}|y-\langle x, w\rangle |^\theta +\frac{1}{n\iota} \log \frac{N(\mathcal{W}, \zeta)}{\eta} \nonumber \\
 	&\leq L_\mathcal{D}(\hat{w}_\iota)+\zeta\tau+\frac{\iota^{\theta-1}}{\theta}\sup_{w\in\mathcal{W}}\mathbb{E}|y-\langle x, w\rangle |^\theta +\frac{1}{n\iota} \log \frac{N(\mathcal{W}, \zeta)}{\eta}, \label{aeq:256}
 \end{align}
 In the following we will show the relation between $L_\mathcal{D}(\hat{w}_\iota)$ and $L_\mathcal{D}(w^*)$. We first show the following lemma: 
 
 \begin{lemma}[\cite{chen2020generalized}]\label{alemma:203}
 With probability at least $1-\eta$, 
  \begin{multline*}
 		L_\mathcal{D}(\hat{w}_\iota)-\hat{L}_\iota (\hat{w}_\iota, D)\leq 2\zeta \tau + \frac{(2\iota)^{\theta-1}}{\theta}\sup_{w\in\mathcal{W}}\mathbb{E}|y-\langle x, w\rangle |^\theta\\ +\frac{(2\iota)^{\theta-1}\zeta^\theta}{\theta}\tau^\theta+\frac{1}{n\iota}\log \frac{N(\mathcal{W}, \zeta)}{\eta}
 \end{multline*}
 \end{lemma}

 By Lemma \ref{alemma:203}, the definition of $\hat{w}_\iota$ we have with probability at least $1-2\eta$
 \begin{align}
 		L_\mathcal{D}(\hat{w}_\iota)&\leq \hat{L}_\iota (\hat{w}_\iota, D)+ 2\zeta \tau + \frac{(2\iota)^{\theta-1}}{\theta}\sup_{w\in\mathcal{W}}\mathbb{E}|y-\langle x, w\rangle |^\theta \notag \\
 	  &\qquad +\frac{(2\iota)^{\theta-1}\zeta^\theta}{\theta}\tau^\theta+\frac{1}{n\iota}\log \frac{N(\mathcal{W}, \zeta)}{\eta} \nonumber\\ 
 		&\leq  L_\mathcal{D}(w^*)+ 2\zeta \tau + \frac{(2\iota)^{\theta-1}}{\theta}\sup_{w\in\mathcal{W}}\mathbb{E}|y-\langle x, w\rangle |^\theta\notag \\
 	  &\qquad +\frac{2(2\iota)^{\theta-1}\zeta^\theta}{\theta}\tau^\theta+\frac{2}{n\iota}\log \frac{N(\mathcal{W}, \zeta)}{\eta}. \label{aeq:206}
 \end{align}
Where the last inequality of (\ref{aeq:206}) is due to the following with probability $1-\eta$, whose proof is the same as in the proof of Lemma \ref{alemma:201} (we omit it here)
\begin{equation}
	\hat{L}_\iota(w^*, D)\leq  L_\mathcal{D}(w^*)+\frac{\iota^{\theta-1}}{\theta}\sup_{w\in\mathcal{W}}\mathbb{E}|y-\langle x, w\rangle|^\theta +\frac{1}{n\iota} \log \frac{1}{\eta}
\end{equation}

 Thus, combing with (\ref{aeq:200}) , Lemma \ref{alemma:200}, (\ref{aeq:256})  and (\ref{aeq:206})  we have with probability at least $1-5\eta$
 \begin{multline}
 		 L_\mathcal{D}(\tilde{w}) \leq  L_\mathcal{D}(w^*)+ 3\zeta \tau + \frac{2(2\iota)^{\theta-1}}{\theta}\sup_{w\in\mathcal{W}}\mathbb{E}|y-\langle x, w\rangle |^\theta\\ +\frac{2(2\iota)^{\theta-1}\zeta^\theta}{\theta}\tau^\theta +\frac{8\log \theta}{n\iota \epsilon}\log \frac{N(\mathcal{W}, \zeta)}{\eta}
 \end{multline}
 Since  $\log N(\mathcal{W}, \zeta)\leq d\log \frac{3\Delta}{\zeta}$, we have 
 \begin{equation}\label{aeq:180}
 	 L_\mathcal{D}(\tilde{w})-L_\mathcal{D}(w^*)\leq {O}(\zeta \tau+ {\iota^{\theta-1}}\tau^\theta+ {\iota^{\theta-1}\zeta^\theta}\tau^\theta + \frac{d}{n\iota \epsilon}\log \frac{1}{\zeta\eta}),
 \end{equation}
 which is due to that $\sup_{w\in\mathcal{W}}\mathbb{E}|y-\langle x, w\rangle |^\theta\leq O(\mathbb{E}|y-\langle x, w^*\rangle |^\theta+\mathbb{E}|\langle x, w^*-w\rangle|^\theta )\leq O(\Delta^\theta \tau^\theta)$. 
 Take $\iota=(\frac{d\log\frac{1}{\zeta}}{n\epsilon})^\frac{1}{\theta}$ and $\zeta=\frac{1}{n}$ we can get the proof. 
 
 For (\ref{eq:180}), we can replace $\tau$ by $\mathbb{E}\|x\|_2$ in (\ref{aeq:180}). Under Assumption \ref{ass:3} and use the inequality $\mathbb{E}\|x\|_2^\theta\leq  \mathbb{E}\|x\|_\theta^\theta=d\tau^\theta$,   we have the result by taking $\iota=O(\frac{1}{\tau}(\frac{\log n}{n\epsilon})^\frac{\theta-1}{\theta})$ and $\zeta=O(\frac{1}{n})$. 
 
 \end{proof}

\bibliographystyle{IEEEtran}
\bibliography{ref}
%

\newpage 
\appendix 
The following omitted proofs haven been showed in  \cite{zhang2018ell_1} and \cite{chen2020generalized}, we include here for self-completeness. 
  \begin{proof}[Proof of Lemma \ref{alemma:100}]
 
First, note that our truncation function $\psi$ satisfies 
$\psi(x)\geq -\log (1-x+\frac{x^2}{2}).$ 
 Thus we have, 
 \begin{align}
 	&\mathbb{E}[\exp(-\sum_{i=1}^n \psi (\iota|y_i-\langle x_i, w \rangle|))] \nonumber \\
 	&\leq \mathbb{E}[|\Pi_{i=1}^n (1-\iota|y_i-\langle x_i, w \rangle|+\frac{\iota^2(y_i-\langle x_i, w \rangle)^2}{2} ) ]\label{aeq:51}\\
 	&=(\mathbb{E}[(1-\iota|y-\langle x, w \rangle|+\frac{\iota^2(y-\langle x, w \rangle)^2}{2} ) ])^n \nonumber \\
 	&=\big(1-\iota L_\mathcal{D}(w)+\frac{\iota^2}{2}\mathbb{E}(y-\langle x, w \rangle)^2\big)^n \nonumber \\ 
 	&\leq \exp\big(n(-\iota L_\mathcal{D}(w)+ \frac{\iota^2}{2}\mathbb{E}(y-\langle x, w \rangle)^2)\big)\label{aeq:52},
 \end{align}
 where (\ref{aeq:51}) is due to the previous inequality and (\ref{aeq:52}) is due the the inequality $1+x\leq e^x$. By the Chernoff's method, we have 
 \begin{align}
 &\text{Pr}\{-\sum_{i=1}^n \psi (\iota|y_i-\langle x_i, w \rangle|)\geq \notag \\
 &n(-\iota L_\mathcal{D}(w)+ \frac{\iota^2}{2}\mathbb{E}(y-\langle x, w \rangle)^2)+\log \frac{1}{\eta}\} \nonumber \\
 	&=\text{Pr}\{\exp(-\sum_{i=1}^n \psi (\iota|y_i-\langle x_i, w \rangle|))\notag \\
 	& \geq \exp\big(n(-\iota L_\mathcal{D}(w)+ \frac{\iota^2}{2}\mathbb{E}(y-\langle x, w \rangle)^2)+\log \frac{1}{\eta}\big)\} \nonumber \\
 	&\leq \frac{\mathbb{E}[\exp(-\sum_{i=1}^n \psi (\iota|y_i-\langle x_i, w \rangle|))] }{\mathbb{E}[\exp\big(n(-\iota L_\mathcal{D}(w)+ \frac{\iota^2}{2}\mathbb{E}(y-\langle x, w \rangle)^2)+\log \frac{1}{\eta}\big)] } \leq \eta.
 \end{align}
 Thus, with probability at least $1-\eta$ we have 
 \begin{align*}
 	&-\sum_{i=1}^n \psi (\iota|y_i-\langle x_i, w \rangle|)\notag \\ &\leq n(-\iota L_\mathcal{D}(w)+ \frac{\iota^2}{2}\mathbb{E}(y-\langle x, w \rangle)^2)+\log \frac{1}{\eta}\\
 	&\leq  n(-\iota L_\mathcal{D}(w)+ \frac{\iota^2}{2}\sup_{w\in \mathcal{W}}\mathbb{E}(y-\langle x, w \rangle)^2)+\log \frac{1}{\eta}.
 \end{align*}
 Take the union bound for all $w\in \tilde{W}_\zeta$ we complete the proof. 
 \end{proof}
 \begin{proof}[Proof of Lemma \ref{alemma:101}]
 First, note the truncation function $\psi$ satisfies 
 	$\psi(x)\leq \log (1+x+\frac{x^2}{2}). $ 
 Thus we have 
  \begin{align}
 	& \mathbb{E}[\exp(\sum_{i=1}^n \psi (\iota|y_i-\langle x_i, w \rangle|))] \notag \\  &\leq \mathbb{E}[|\Pi_{i=1}^n (1+\iota|y_i-\langle x_i, w \rangle|+\frac{\iota^2(y_i-\langle x_i, w \rangle)^2}{2} ) ]\label{aeq:54}\\
 	&=(\mathbb{E}[(1+\iota|y-\langle x, w \rangle|+\frac{\iota^2(y-\langle x, w \rangle)^2}{2} ) ])^n \nonumber \\
 	&=\big(1+\iota L_\mathcal{D}(w)+\frac{\iota^2}{2}\mathbb{E}(y-\langle x, w \rangle)^2\big)^n \nonumber \\ 
 	&\leq \exp\big(n(\iota L_\mathcal{D}(w)+ \frac{\iota^2}{2}\mathbb{E}(y-\langle x, w \rangle)^2)\big),\label{aeq:55}
 \end{align}
 where (\ref{aeq:54}) is due to the previous inequality and (\ref{aeq:55}) is due the the inequality $1+x\leq e^x$. By the Chernoff's method, 
  \begin{align}
 &\text{Pr}\{\sum_{i=1}^n \psi (\iota|y_i-\langle x_i, w \rangle|)\notag \\
 & \geq n(\iota L_\mathcal{D}(w)+ \frac{\iota^2}{2}\mathbb{E}(y-\langle x, w \rangle)^2)+\log \frac{1}{\eta}\} \nonumber \\
 	&=\text{Pr}\{\exp(\sum_{i=1}^n \psi (\iota|y_i-\langle x_i, w \rangle|))\notag \\
 	& \geq \exp\big(n(\iota L_\mathcal{D}(w)+ \frac{\iota^2}{2}\mathbb{E}(y-\langle x, w \rangle)^2)+\log \frac{1}{\eta}\big)\} \nonumber \\
 	&\leq \frac{\mathbb{E}[\exp(\sum_{i=1}^n \psi (\iota|y_i-\langle x_i, w \rangle|))] }{\mathbb{E}[\exp\big(n(\iota L_\mathcal{D}(w)+ \frac{\iota^2}{2}\mathbb{E}(y-\langle x, w \rangle)^2)+\log \frac{1}{\eta}\big)] } \leq \eta.\nonumber
 \end{align}
 Thus, with probability at least $1-\eta$ we have 
 \begin{align*}
 	& \sum_{i=1}^n \psi (\iota|y_i-\langle x_i, w \rangle|)\\
 	& \leq n(\iota L_\mathcal{D}(w)+ \frac{\iota^2}{2}\mathbb{E}(y-\langle x, w \rangle)^2)+\log \frac{1}{\eta}\\
 	&\leq  n(\iota L_\mathcal{D}(w)+ \frac{\iota^2}{2}\sup_{w\in \mathcal{W}}\mathbb{E}(y-\langle x, w \rangle)^2)+\log \frac{1}{\eta}.
 \end{align*}
 Take the union bound for all $w\in \tilde{W}_\zeta$ we complete the proof. 
 \end{proof}
   \begin{proof}[Proof of Lemma \ref{alemma:200}]
 
First, note that our truncation function $\psi$ satisfies
 	$\psi(x)\geq -\log (1-x+\frac{|x|^\theta}{\theta}). $ 
  Thus we have, 
 \begin{align}
 	&\mathbb{E}[\exp(-\sum_{i=1}^n \psi (\iota|y_i-\langle x_i, w \rangle|))]\notag \\ &\leq \mathbb{E}[|\Pi_{i=1}^n (1-\iota|y_i-\langle x_i, w \rangle|+\frac{\iota^\theta|y_i-\langle x_i, w \rangle|^\theta}{\theta} ) ]\label{aeq:201}\\
 	&=(\mathbb{E}[(1-\iota|y-\langle x, w \rangle|+\frac{\iota^\theta|y_i-\langle x_i, w \rangle|^\theta}{\theta} )])^n \nonumber \\
 	&=\big(1-\iota L_\mathcal{D}(w)+\frac{\iota^\theta}{\theta}\mathbb{E}|y-\langle x, w \rangle|^\theta\big)^n \nonumber \\ 
 	&\leq \exp\big(n(-\iota L_\mathcal{D}(w)+ \frac{\iota^\theta}{\theta}\mathbb{E}|y-\langle x, w \rangle|^\theta)\big),\label{aeq:202}
 \end{align}
 where (\ref{aeq:201}) is due to the previous inequality and (\ref{aeq:202}) is due the the inequality $1+x\leq e^x$. By the Chernoff's method, we have 
 \begin{align}
 &\text{Pr}\{-\sum_{i=1}^n \psi (\iota|y_i-\langle x_i, w \rangle|)\notag \\
 &\geq n(-\iota L_\mathcal{D}(w)+ \frac{\iota^\theta}{\theta}\mathbb{E}|y-\langle x, w \rangle|^\theta)+\log \frac{1}{\zeta}\} \nonumber \\
 	&=\text{Pr}\{\exp(-\sum_{i=1}^n \psi (\iota|y_i-\langle x_i, w \rangle|)) \notag\\
 	&\geq \exp\big(n(-\iota L_\mathcal{D}(w)+ \frac{\iota^\theta}{\theta}\mathbb{E}|y-\langle x, w \rangle|^\theta)+\log \frac{1}{\eta}\big)\} \nonumber \\
 	&\leq \frac{\mathbb{E}[\exp(-\sum_{i=1}^n \psi (\iota|y_i-\langle x_i, w \rangle|))] }{\mathbb{E}[\exp\big(n(-\iota L_\mathcal{D}(w)+ \frac{\iota^\theta}{\theta}\mathbb{E}|y-\langle x, w \rangle|^\theta+\log \frac{1}{\eta}\big)] } \leq \eta.
 \end{align}
 Thus, with probability at least $1-\eta$ we have 
 \begin{align*}
 	&-\sum_{i=1}^n \psi (\iota|y_i-\langle x_i, w \rangle|) \\ &\leq n(-\iota L_\mathcal{D}(w)+\frac{\iota^\theta}{\theta}\mathbb{E}|y-\langle x, w \rangle|^\theta)+\log \frac{1}{\eta}\\
 	&\leq  n(-\iota L_\mathcal{D}(w)+ \frac{\iota^\theta}{\theta}\sup_{w\in \mathcal{W}}\mathbb{E}|y-\langle x, w \rangle|^\theta)+\log \frac{1}{\eta}.
 \end{align*}
 Take the union bound for all $w\in \tilde{W}_\zeta$ we complete the proof. 
 \end{proof}
  \begin{proof}[Proof of Lemma \ref{alemma:201}]
 First, note the truncation function $\psi$ satisfies 
 	$\psi(x)\leq \log (1+x+\frac{|x|^\theta}{\theta}). $ 
 Thus we have 
  \begin{align}
 	&\mathbb{E}[\exp(\sum_{i=1}^n \psi (\iota|y_i-\langle x_i, w \rangle|))] \notag 
 	\\ &\leq \mathbb{E}[|\Pi_{i=1}^n (1+\iota|y_i-\langle x_i, w \rangle|+\frac{\iota^\theta|y_i-\langle x_i, w \rangle|^\theta}{\theta} ) ]\label{aeq:204}\\
 	&=(\mathbb{E}[(1+\iota|y-\langle x, w \rangle|+\frac{\iota^\theta|y_i-\langle x_i, w \rangle|^\theta}{\theta} ) ])^n \nonumber \\
 	&=\big(1+\iota L_\mathcal{D}(w)+\frac{\iota^\theta}{\theta}\mathbb{E}|y-\langle x, w \rangle|^\theta\big)^n \nonumber \\ 
 	&\leq \exp\big(n(\iota L_\mathcal{D}(w)+ \frac{\iota^\theta}{\theta}\mathbb{E}|y-\langle x, w \rangle|^\theta\big)\label{aeq:2090},
 \end{align}
 where (\ref{aeq:204}) is due to the previous inequality and (\ref{aeq:2090}) is due the the inequality $1+x\leq e^x$. By the Chernoff's method, we have 
  \begin{align}
 &\text{Pr}\{\sum_{i=1}^n \psi (\iota|y_i-\langle x_i, w \rangle|)\notag \\
 &\geq n(\iota L_\mathcal{D}(w)+ \frac{\iota^\theta}{\theta}\mathbb{E}|y-\langle x, w \rangle|^\theta)+\log \frac{1}{\eta}\} \nonumber \\
 	&=\text{Pr}\{\exp(\sum_{i=1}^n \psi (\iota|y_i-\langle x_i, w \rangle|)) \notag \\
 	&\geq \exp\big(n(\iota L_\mathcal{D}(w)+ \frac{\iota^\theta}{\theta}\mathbb{E}|y-\langle x, w \rangle|^\theta)+\log \frac{1}{\eta}\big)\} \nonumber \\
 	&\leq \frac{\mathbb{E}[\exp(\sum_{i=1}^n \psi (\iota|y_i-\langle x_i, w \rangle|))] }{\mathbb{E}[\exp\big(n(\iota L_\mathcal{D}(w)+ \frac{\iota^\theta}{\theta}\mathbb{E}|y-\langle x, w \rangle|^\theta)+\log \frac{1}{\eta}\big)] } \leq \eta.\nonumber
 \end{align}
 Thus, with probability at least $1-\eta$ we have 
 \begin{align*}
 	&\sum_{i=1}^n \psi (\iota|y_i-\langle x_i, w \rangle|) \leq n(\iota L_\mathcal{D}(w)+ \frac{\iota^\theta}{\theta}\mathbb{E}|y-\langle x, w \rangle|^\theta)+\log \frac{1}{\eta}\\
 	&\leq  n(\iota L_\mathcal{D}(w)+ \frac{\iota^\theta}{\theta}\sup_{w\in \mathcal{W}}\mathbb{E}|y-\langle x, w \rangle|^\theta)+\log \frac{1}{\eta}.
 \end{align*}
 Take the union bound for all $w\in \tilde{W}_\zeta$ we complete the proof. 
 \end{proof}
  \begin{proof}[Proof of Lemma \ref{alemma:203}]
 As we mentioned before, there exists a $\tilde{w}_\iota \in \tilde{W}_\zeta$ such that $\|\tilde{w}_\iota-\hat{w}_\iota\|_2\leq \zeta$. This implies that 
 \begin{multline*}
 	|y_i-\langle x_i, \hat{w}_\iota\rangle| \geq |y_i-\langle x_i, \tilde{w}_\iota\rangle|-|\langle x_i, \tilde{w}_\iota-\hat{w}_\iota| \\ \geq |y_i-\langle x_i, \tilde{w}_\iota\rangle|-\zeta \|x_i\|_2. 
 \end{multline*}
 Since $\psi$ is non-decreasing, this implied that 
 \begin{multline*}
 	\hat{L}_\alpha (\hat{w}_\iota, D)=\frac{1}{n\iota}\sum_{i=1}^n \psi(\iota|y_i-\langle x_i, \hat{w}_\iota\rangle|)\\ \geq \frac{1}{n\iota}\sum_{i=1}^n \psi(\iota|y_i-\langle x_i, \tilde{w}_\iota\rangle|-\iota\zeta \|x_i\|_2). 
 \end{multline*}
 We then proof the following lemma 
 \begin{lemma}\label{alemma:204}
 	For any $w\in \tilde{W}_\zeta$, with probability at least $1-\eta$, the following inequality holds: 
 	\begin{multline*}
 		-\frac{1}{n\iota}\sum_{i=1}^n \psi(\iota|y_i-\langle x_i, w \rangle|-\iota\zeta \|x_i\|_2)\leq -L_\mathcal{D}(w)+\zeta \tau\\ + \frac{(2\iota)^{\theta-1}}{\theta}\sup_{w\in\mathcal{W}}\mathbb{E}|y-\langle x, w\rangle |^\theta+ \frac{(2\iota)^{\theta-1}\zeta^\theta}{\theta}\tau^\theta+\frac{1}{n\iota}\log \frac{N(\mathcal{W}, \zeta)}{\eta}. 
 	\end{multline*}
 	
 \end{lemma}
 \begin{proof}[Proof of Lemma \ref{alemma:204} ] 
 The idea of proof is almost the same as in the proof of Lemma \ref{alemma:201}. First, note that our truncation function $\psi$ satisfies the following 
 	$\psi(x)\geq -\log (1-x+\frac{|x|^\theta}{\theta}).$ 
 Thus we have, 
 \begin{align}
 	& \mathbb{E}[\exp(-\sum_{i=1}^n \psi (\iota|y_i-\langle x_i, w\rangle|-\iota\zeta \|x_i\|_2))] \nonumber \\
 	 &\leq \mathbb{E}[|\Pi_{i=1}^n (1-\iota|y_i-\langle x_i, w \rangle|\notag \\
 	  &\qquad +\iota\zeta \|x_i\|_2+\frac{\iota^\theta(|y_i-\langle x_i, w \rangle|-\zeta\|x_i\|_2)^\theta}{\theta}) ]\\
 	&=(\mathbb{E}[(1-\iota|y-\langle x, w \rangle|\notag \\
 	  &\qquad +\iota\zeta \|x\|_2+\frac{\iota^\theta(|y-\langle x, w \rangle|-\zeta\|x\|_2)^\theta}{\theta} )])^n \nonumber \\
 	&=\big(1-\iota L_\mathcal{D}(w)+\iota\zeta\mathbb{E}\|x\|_2+\frac{\iota^\theta}{\theta}\mathbb{E}(|y-\langle x, w \rangle|-\zeta\|x\|_2)^\theta
 	\big)^n \nonumber \\ 
 	&\leq \exp\big(n(-\iota L_\mathcal{D}(w)+ \frac{\iota^\theta 2^{\theta-1}}{\theta}\mathbb{E}|y-\langle x, w \rangle|^\theta\notag \\
 	  &\qquad  +\frac{\iota^\theta 2^{\theta-1}\zeta^\theta}{\theta}\mathbb{E}\|x\|_2^\theta )\big).\end{align}
 By the Chernoff's method, we have with probability at least $1-\eta$
 \begin{multline*}
 	-\sum_{i=1}^n \psi (\iota|y_i-\langle x_i, w\rangle|-\iota\zeta \|x_i\|_2)\\ \geq -\iota L_\mathcal{D}(w)+ \frac{\iota^\theta 2^{\theta-1}}{\theta}\mathbb{E}|y-\langle x, w \rangle|^\theta+\frac{\iota^\theta 2^{\theta-1}\zeta^\theta}{\theta}\mathbb{E}\|x\|_2^\theta+ \log \frac{1}{\eta}. 
 \end{multline*}
 Take the union and then we complete the proof. 
 \end{proof}
 Thus by Lemma \ref{alemma:204} we have with probability at least $1-\eta$
 \begin{align}
 		&\hat{L}_\alpha (\hat{w}_\iota, D)=\frac{1}{n\iota}\sum_{i=1}^n \psi(\iota|y_i-\langle x_i, \hat{w}_\iota\rangle|)\notag \\
 	  &\geq \frac{1}{n\iota}\sum_{i=1}^n \psi(\iota|y_i-\langle x_i, \tilde{w}_\iota\rangle|-\iota\zeta \|x_i\|_2) \nonumber \\
 		&\geq L_\mathcal{D}(\tilde{w}_\iota)-\zeta \tau-\frac{(2\iota)^{\theta-1}}{\theta}\sup_{w\in\mathcal{W}}\mathbb{E}|y-\langle x, w\rangle |^\theta \notag \\
 	  &\qquad  - \frac{(2\iota)^{\theta-1}\zeta^\theta}{\theta}\tau^\theta-\frac{1}{n\iota}\log \frac{N(\mathcal{W}, \zeta)}{\eta} \nonumber \\
 		& \geq  L_\mathcal{D}(\hat{w}_\iota)-2\zeta \tau-\frac{(2\iota)^{\theta-1}}{\theta}\sup_{w\in\mathcal{W}}\mathbb{E}|y-\langle x, w\rangle |^\theta \notag \\
 	  &\qquad - \frac{(2\iota)^{\theta-1}\zeta^\theta}{\theta}\tau^\theta-\frac{1}{n\iota}\log \frac{N(\mathcal{W}, \zeta)}{\eta},   \label{aeq:205}
 \end{align}
 where (\ref{aeq:205}) is due to 
  \begin{align*}
 	& L_\mathcal{D}(\tilde{w}_\iota)-L_\mathcal{D}(\hat{w}_\iota)=\mathbb{E}[|y-\langle x, \tilde{w}_\iota \rangle|-|y-\langle x, \hat{w}_\iota \rangle|] \notag \\
 	  & \leq \mathbb{E} |\langle x, \tilde{w}_\iota-\hat{w}_\iota\rangle|\leq \zeta\mathbb{E}\|x\|_2\leq \zeta \tau
 \end{align*}
 \end{proof}
\end{document}